\documentclass[11pt]{article}

\usepackage{arxiv}

\usepackage[utf8]{inputenc} 
\usepackage[T1]{fontenc}    
\usepackage{hyperref}       
\usepackage{url}            
\usepackage{booktabs}       
\usepackage{amsfonts}       
\usepackage{nicefrac}       
\usepackage{microtype}      
\usepackage{lipsum}
\usepackage{graphicx}
\usepackage{amsmath,amsthm,amssymb,amsfonts}%

\usepackage{color}
\usepackage{xcolor}
\usepackage[frozencache, cachedir=minted-cache]{minted}

\newtheorem{thm}{Theorem}
\newtheorem*{thm*}{Theorem}
\newtheorem{prop}[thm]{Proposition}
\newtheorem{lem}[thm]{Lemma}
\newtheorem*{lem*}{Lemma}

\theoremstyle{remark}
\newtheorem*{remark}{Remark}

\newcommand\cC{\mathcal C}
\newcommand\cCone{\mathcal{C}^{(1)}}
\newcommand\tcCone{\tilde{\mathcal{C}}^{(1)}}
\newcommand\hcCone{\hat{\mathcal{C}}^{(1)}}
\newcommand\cCtwo{\mathcal{C}^{(2)}}
\newcommand\cCthree{\mathcal{C}^{(3)}}

\newcommand\cE{\mathcal E}
\newcommand{\bigO}{\mathcal{O}}

\begin{document}
\title{A Pragmatic Method for Comparing Clusterings with Overlaps and Outliers}
\author{
    Ryan DeWolfe\thanks{Department of Mathematics, Toronto Metropolitan University, Toronto, Canada; \texttt{ryan.dewolfe@torontomu.ca}},
    ~Pawe\l{} Pra\l{}at\thanks{Department of Mathematics, Toronto Metropolitan University, Toronto, Canada; \texttt{pralat@torontomu.ca}},
    ~Fran\c{c}ois Th\'{e}berge\thanks{Tutte Institute for Mathematics and Computing, Ottawa, Canada; \texttt{theberge@ieee.org}}
    }
\date{}
\maketitle

\begin{abstract}
    Clustering algorithms are an essential part of the unsupervised data science ecosystem, and extrinsic evaluation of clustering algorithms requires a method for comparing the detected clustering to a ground truth clustering.
    In a general setting, the detected and ground truth clusterings may have outliers (objects belonging to no cluster), overlapping clusters (objects may belong to more than one cluster), or both, but methods for comparing these clusterings are currently undeveloped.
    In this note, we define a pragmatic similarity measure for comparing clusterings with overlaps and outliers, show that it has several desirable properties, and experimentally confirm that it is not subject to several common biases afflicting other clustering comparison measures.
\end{abstract}

\section{Introduction}
Clustering~\cite{clusteringhandbook} (or community detection~\cite{cdreview}) is one of the fundamental problems in unsupervised machine learning, and typically involves finding groups of ``similar'' data within a dataset.
Many clustering algorithms often only consider partitions of the data (a clustering where each object belongs to exactly one cluster), but several recent algorithms allow for outliers~\cite{hdbscan,wcc}, overlapping clusters~\cite{overlapreview}, or both~\cite{hslc,oslom}.
One method to evaluate the performance of a clustering algorithm is extrinsic clustering validity: the algorithm is judged by its ability to find a known ground truth clustering in the data.
Thus, extrinsic validity benchmarks require a method to compare clusterings~\cite{comparisonreview}.

In this note, we use set-matching~\cite{setmatching} and the $F^*$ score~\cite{F*} to define a pragmatic method for comparing clusterings that could have overlaps or outliers.
The method is simple, computationally efficient, and exhibits several intuitive properties one would expect in a comparison measure.

\section{Related Work}\label{sec:related}
Comparing partitions is a well-studied problem, and there are a wide variety of measures that are used~\cite{comparisonreview}, which can be broadly classified as either pair-counting, information theoretic, or set-matching~\cite{setmatching}.
In this section, we review some extensions to overlapping clusterings, although to the best of our knowledge, none have explicitly considered outliers.

The \textit{Omega} index~\cite{omega} extends the pair-counting Adjusted Rand Index (ARI)~\cite{ari} by extending the contingency table from agree versus disagree to counting the number of clusters in common.
This extension is reasonable and is equivalent to the ARI for partitions, but is quite harsh since it does not detect any similarity if a pair of objects does not share the exact same number of clusters in both clusterings.
Furthermore, the non-adjusted version of the Omega index suffers from the same issues as the non-adjusted Rand index: values tend to be large, even for random clusterings.
The adjustment proposed implicitly assumes a null-model for generating random clusterings, the choice of which has been shown to have an effect for comparing partitions~\cite{partition_models}, and generally requires careful consideration of the assumptions when extending beyond partitions~\cite{fuzzy_models}.

Lancichinetti et al.~\cite{onmi} proposed the overlapping Normalized Mutual Information (oNMI), an extension of the information theoretic Normalized Mutual Information~\cite{ami} using a set matching framework.
The oNMI (and NMI) requires an adjustment term to normalize values across datasets, and McDaid et al.~\cite{onmi2} note that the original adjustment leads to some unintuitive behaviour, and suggests an alternative.
In fact, there are 5 possible adjustment terms and, similar to the previously mentioned null-models, any choice must be carefully considered.
In our experiments, we test both suggested adjustments, referring to that of Lancichinetti et al.~\cite{onmi} as $\mathrm{oNMI_{LFK}}$ and that of McDaid et al.~\cite{onmi2} as $\mathrm{oNMI_{MGH}}$.

Finally, by shifting the focus from clusters to objects, the Element-Centric Similarity (ECS)~\cite{ecc} is a new class of method.
The method constructs a \textit{cluster-induced element graph} by placing a directed and non-uniformly weighted clique over each cluster (with weights chosen so that the degree of each object is split evenly across its clusters), and comparing the personalized pagerank vectors of each object.
This method is intuitive and allows for interesting object-level agreement analysis, but  has several drawbacks for comparing clusterings.
Critically, the ECS may assign maximum similarity to non-identical clusterings; for example, a single cluster that contains all objects compared to any uniform Balanced Incomplete Block Design (BIBD) over the objects will achieve maximum similarity.
A BIBD is a clustering where each pair of objects appears in the same number of clusters and each cluster has the same size.
These symmetries create a cluster induced element graph that is a clique with equal weights, which is identical (up to a constant scaling, which is ignored by the personalized pagerank computation) to the cluster-induced element graph of a single cluster containing all objects.
Furthermore, the ECS is computationally difficult for large datasets with overlapping clusters, and the natural extension would treat outliers the same as clusters of size one.

Finally, we briefly mention the \textit{problem of matching}~\cite{vi} that affects set-matching measures~\cite{setmatching} like the oNMI and our proposed measure.
The problem of matching refers to the indifference of a measure to the part of each cluster that is unmatched.
In particular, if the unmatched portion of a cluster is concentrated in one other cluster, then it is somehow less disrupted than if the unmatched portion was spread across many clusters.
However, while it may be less disrupted in an edit distance sense (and only if moving multiple objects between clusters costs less than moving each object individually), we disagree that intuition dictates that these disruptions should not be equally dissimilar to the original. 
In fact, by measuring the ability of each cluster to find a match in the other clustering, a set-matching measure provides interesting cluster level information that could be useful for detecting where the clusterings disagree, possibly complementing the object level information available from the ECS~\cite{ecc}.

\section{The Method}
For a set of objects $X = \{x_1, x_2, \ldots, x_n\}$, we define a \textit{clustering} as set of non-empty subsets of $X$.
We call each element of a clustering a \textit{cluster}.
Note that with this definition, each cluster in a clustering is distinct, and the empty set cannot be a cluster.
However, we allow the clustering itself to be the empty set.

To define methods for comparing clusterings $\cCone$ and $\cCtwo$, we first define some methods for comparing two individual clusters.
A cluster similarity function is a map between two subsets of $X$ and a value in $[0,1]$ that measures some notion of similarity (with $1$ denoting maximum similarity).
The standard precision ($p$), recall ($r$), and $F_1$ similarity functions between two clusters $C_i, C_j \subseteq X$ are defined as:
\begin{gather*}
    p(C_i, C_j) = \frac{|C_i \cap C_j|}{|C_i|},\\
    r(C_i, C_j) = \frac{|C_i \cap C_j|}{|C_j|},\\
    F_1(C_i, C_j) = \frac{2}{1 / p(C_i, C_j) + 1 / r(C_i, C_j)}.
\end{gather*}
Precision and recall both have nice interpretations as the probability of being in a cluster conditional on being in other.
The $F_1$ score is very popular to combine the precision and recall into a single score, but has gathered criticism for being pragmatic but ultimately unjustified~\cite{F-review}.
In~\cite{F*}, the authors argue that another combination of precision and recall, which they call the $F^*$ score, is more interpretable.
Furthermore, they show that $F^*$ is a monotonic transformation of $F_1$, which means any ranking of similarities using $F_1$ will be the same when using $F^*$.
The $F^*$ score is equivalent to the well-known Jaccard index for set similarity, and is defined as
\begin{equation*}
F^*(C_i, C_j) = \frac{F_1(C_i, C_j)}{2 - F_1(C_i, C_j)} = \frac{|C_i \cap C_j|}{|C_i \cup C_j|}.
\end{equation*}
For the rest of the paper, we will use $F^*$.
It could be replaced by any of the above (or other) set similarity measures, but it is easy to work with and will serve as a general comparison measure.

Next, we define methods for measuring how well each cluster in $\cCone$ can find a matching cluster in $\cCtwo$.
For a cluster $C_i \in \cCone$, we find a matching cluster by taking the maximum similarity when comparing to each cluster in $\cCtwo$:
\begin{equation*}
    F^*(C_i, \cCtwo) = \max_{C_j \in \cCtwo} \big\{ F^*(C_i, C_j) \big\}.
\end{equation*}
We define the edge case $F^*(C_i, \emptyset) = 0$ since $C_i$ cannot find a cluster with any similarity.

Then, to summarize the match quality from each cluster, we compute the weighted average similarity.
\begin{equation*}
    \hat{F}^{*}_w(\cCone, \cCtwo) = \sum_{C_i \in \cCone} \frac{|C_i|}{\sum_{C_k \in \cCone} |C_k|} F^*(C_i, \cCtwo).
\end{equation*}
This function is a measure of how well each cluster in $\cCone$ can find a match in $\cCtwo$, but not necessarily the reverse.
This asymmetric similarity measure could be useful in domains where the performance of a clustering algorithm should be measured by its ability to find clusters that actually exist, but where we are not concerned with complete coverage; for example, in some instances of exploratory data analysis.

Of course, we can also find the best match for each cluster from $\cCtwo$ in $\cCone$ with $\hat{F}^*_w(\cCtwo, \cCone)$.
Finally, we use a linear combination with equal weights to combine both directions into a single symmetric similarity score;
\begin{equation*}
    F^*_w(\cCone, \cCtwo) = 0.5 \hat{F}^{*}_w(\cCone, \cCtwo) + 0.5\hat{F}^{*}_w(\cCtwo, \cCone).
\end{equation*}

The similarity measure $F^*_w$ serves as a simple score to evaluate how similar the clusters of two clusterings are, but crucially does not consider objects that are not part of any cluster.
We call an object $x_i \in X$ an \textit{outlier} with respect to a clustering $\cC$ if $x \not \in \bigcup_{C \in \cC} C$.
For cases where most of the objects are not outliers, $F^*_w$ might be acceptable, but we can add a term for matching the outliers between clusterings.
Let $O^{(1)}$ be the set of outliers with respect to $\cCone$, and likewise for $O^{(2)}$ with $\cCtwo$.
The outlier aware versions of $\hat{F}^*_w$ and $F^*_w$ are defined as
\begin{equation*}
    \hat{F}^{*}_{wo}(\cCone, \cCtwo) = \frac{|O^{(1)}|}{|X|} F^*(O^{(1)}, O^{(2)}) + \frac{|X \setminus O^{(1)}|}{|X|} \hat{F}^{*}_w(\cCone, \cCtwo)
\end{equation*}
and
\begin{equation*}
    F^*_{wo}(\cCone, \cCtwo) = 0.5\hat{F}^*_{wo}(\cCone, \cCtwo) + 0.5\hat{F}^*_{wo}(\cCtwo, \cCone).
\end{equation*}

\section{Properties}
In this section, we analyze some basic properties of the proposed $F^*_{wo}$ score, and establish that it is reasonable in matching the intuition and general use-case of comparing clustering algorithms.
First, we state a few properties that are absolutely required for an intuitive similarity measure, the proofs of which are elementary and are omitted for brevity.

\begin{prop}
    The proposed measure $F^*_{wo}$ has the following properties.
    \begin{enumerate}
        \item $F^*_{wo}$ is label invariant.
        \item $F^*_{wo}(\cCone, \cCtwo) \in [0,1]$, and $F^*_{wo}(\cCone, \cCtwo) = 1$ if and only if $\cCone = \cCtwo$.
        \item $F^*_{wo}$ is symmetric.
        \item Single object clusters are not the same as outliers.
    \end{enumerate}
\end{prop}

The first property is the essential difference between clustering and multi-class classification tasks.
Implementations generally assign a label to each cluster (often using the natural numbers $1$ to $|\cC|$), but such a labelling is arbitrary and should not affect the similarity to another clustering.
The second property is called \textit{normalization}~\cite{ami}, and it is defined as a bounded range for the measure independent of the particular clusterings under consideration.
A normalized measure means that the output is on the same scale regardless of the clusterings, and scores are comparable even if the clusterings have a different number of clusters or objects.
The final property is the key differentiation of our method with those presented in Section~\ref{sec:related}, in the sense that the obvious extensions of the existing methods would treat outliers identically to single object clusters.

The complement of some similarity measures (that is, $1 - s$ for a similarity measure $s$) for partitions are metrics, and it has been argued that a metric structure better matches our intuitions about similarity \cite{vi}.
However, in practice, many of the most popular similarity measures, such as the ARI~\cite{ari}, NMI~\cite{nmi} (using the arithmetic mean, the default in the popular python machine learning library scikit-learn), and AMI~\cite{ami}, in addition to the Omega~\cite{omega}, oNMI~\cite{onmi}, and ECS~\cite{ecc} from Section~\ref{sec:related}, are not the complement of a metric.
Thus, even though being the complement of a metric is desirable, we conclude that this property is not absolutely necessary for a pragmatic similarity measure.

\begin{prop} \label{prop:metric}
    The complement of the proposed measure, $1 - F^*_{wo}$, is not a metric.
\end{prop}
\begin{proof}
    We provide a counterexample for the triangle inequality.
    Let $M(\cCone, \cCtwo) = 1 - F^*_{wo}(\cCone, \cCtwo)$.
    Let $X = \{1,2,3\}$, and consider the following three clusterings:
    \begin{equation*}
        \cCone = \{ \{1,2,3\} \}, \quad \cCtwo = \{ \{1\}, \{1,2,3\} \}, \quad \cCthree =  \{ \{1\}, \{2,3\} \}
    \end{equation*}
    We compute
    \begin{align*}
        M(\cCone, \cCtwo) &= 1 - \left( 0.5\left(1\right) + 0.5\left(\frac{1}{4}\frac{1}{3} + \frac{3}{4}1\right) \right) = \frac{1}{12},\\
        M(\cCtwo, \cCthree) &= 1 - \left( 0.5\left(\frac{1}{4}1 + \frac{3}{4}\frac{2}{3}\right) + 0.5\left(\frac{1}{3}1 + \frac{2}{3}\frac{2}{3}\right) \right) = \frac{17}{72},\\
        M(\cCone, \cCthree) &= 1 - \left( 0.5\left(\frac{2}{3}\right) + 0.5\left(\frac{1}{3}\frac{1}{3} + \frac{2}{3}\frac{2}{3}\right) \right) = \frac{7}{18},
    \end{align*}
    and find that
    \begin{equation*}
        M(\cCone, \cCtwo) + M(\cCtwo, \cCthree) = \frac{1}{12} + \frac{17}{72} = \frac{23}{72} < \frac{7}{18} = M(\cCone, \cCthree),
    \end{equation*}
    and the triangle inequality does not hold.
\end{proof}
\begin{remark}
    Since none of the clusterings used in the proof of Proposition~\ref{prop:metric} contain outliers, this counterexample shows that $1-F^*_w$ is also not a metric.
\end{remark}

Finally, we examine the effect of a small change to one of the clusterings being compared.
\begin{thm}\label{thm:robust}
    Let $\cCone$ and $\cCtwo$ be clusterings of $X$, and let $\tcCone$ be a clustering created from $\cCone$ by moving one object, $x$, into or out of one cluster.
    Let $B = \sum_{C \in \cCtwo} |C|$.
    If $|\cCone| = |\tcCone|$ (no clusters were created or destroyed), say $x$ was added or removed from $C_1 \in \cCone$.
    Let $\gamma = |\{C \in \cCtwo : C \text{ matched with } C_1 \text{ or } x \in C\}|$.
    Then
    \begin{align*}
    &\left| F^*_{wo}(\cCone, \cCtwo) - F^*_{wo}(\tcCone, \cCtwo) \right|\\
    &\qquad = \bigO \left( \frac{1}{|X|} \max \left\{ 1, \frac{|X \setminus O^{(2)}| \gamma}{B} \right\} \right).
    \end{align*}
    Otherwise, let $\Gamma = |\{C \in \cCtwo : x \in C\}|$, and
    \begin{equation*}
            \left| F^*_{wo}(\cCone, \cCtwo) - F^*_{wo}(\tcCone, \cCtwo) \right| = \bigO \left( \frac{1}{|X|} \max \left\{ 1, \frac{|X \setminus O^{(2)}| \Gamma}{B} \right\} \right).
    \end{equation*}
\end{thm}
The proof can be found in Appendix~\ref{app:proof}.
With Theorem~\ref{thm:robust}, we confirm that, subject to some mild constraints on the clusterings, moving one element into or out of a cluster has only a small effect on the similarity to a reference clustering.
In both cases, the right-hand side can be improved to $\bigO(\frac{1}{|X|})$ with additional mild constraints on $\cCtwo$ that prevent an object from being in too many clusters.
Furthermore, in the proof of Theorem~\ref{thm:robust}, we use the following Lemma that is of independent interest.

\begin{lem}\label{lem:two}
    Let $A = \sum_{C \in \cCone} |C|$.
    \begin{equation*}
        \left| \hat{F}^*_w(\cCone, \cCtwo) - \hat{F}^*_w(\tcCone, \cCtwo) \right| = \bigO\left(\frac{1}{A}\right).
    \end{equation*}
\end{lem}

\begin{remark}
    Lemma~\ref{lem:two} implies that if $A$ is sufficiently large, then $\hat{F}^*_w(\cCone, \cCtwo)$ is also robust to a small change in $\cCone$.
    Thus, this is a good choice for evaluating exploratory data analysis algorithms when comparing to a ground truth without outliers.
\end{remark}

\section{Experiments}\label{sec:experiments}

\subsection{Intuitive Properties}
In this section, we recreate and extend an experiment from Gates et al.~\cite{ecc} to check for common biases in the behaviour of $F^*_{wo}$.
The original experiment consists of three series of clustering comparisons that have an expected intuitive behaviour.
None of the scenarios have outliers nor overlapping clusters, but a good measure should still match our intuitions for partitions.

For the first experiment, a clustering is created from $1024$ objects with 32 clusters of each of size $32$.
The second clusterings starts as a copy of the first, but a fraction of its objects have their cluster shuffled.
We expect that as the fraction shuffled increases, the similarity should be monotone decreasing to some non-zero value.

In the second scenario, the first clustering is again $1024$ objects split into $32$ clusters of size $32$.
The second clustering starts with as a different clustering (randomly selected) with $32$ clusters of size $32$, and evolves by reassigning a random object to a new cluster chosen proportional to the cluster sizes.
As the process proceeds, we expect the cluster sizes to become more heterogeneous (a few large clusters, many small clusters) and we again expect the similarity to be monotone decreasing to some non-zero value.

Finally, we test for bias in the number of clusters.
The first clustering has $1024$ objects split into $8$ clusters of size $128$, and the second clustering is chosen uniformly from partitions of $1024$ objects into $k$ equally sized clusters.
As $k$ increases, we expect the similarity to be monotone decreasing, and for the similarity to be almost $0$ in the extreme case of $1024$ clusters of size $1$.

As seen in Figure \ref{fig:experiment}, the behaviour of $F^*_{wo}$ matches our intuition in each experiment.

\begin{figure*}[!t]
    \centering
    \includegraphics[width=0.9\linewidth]{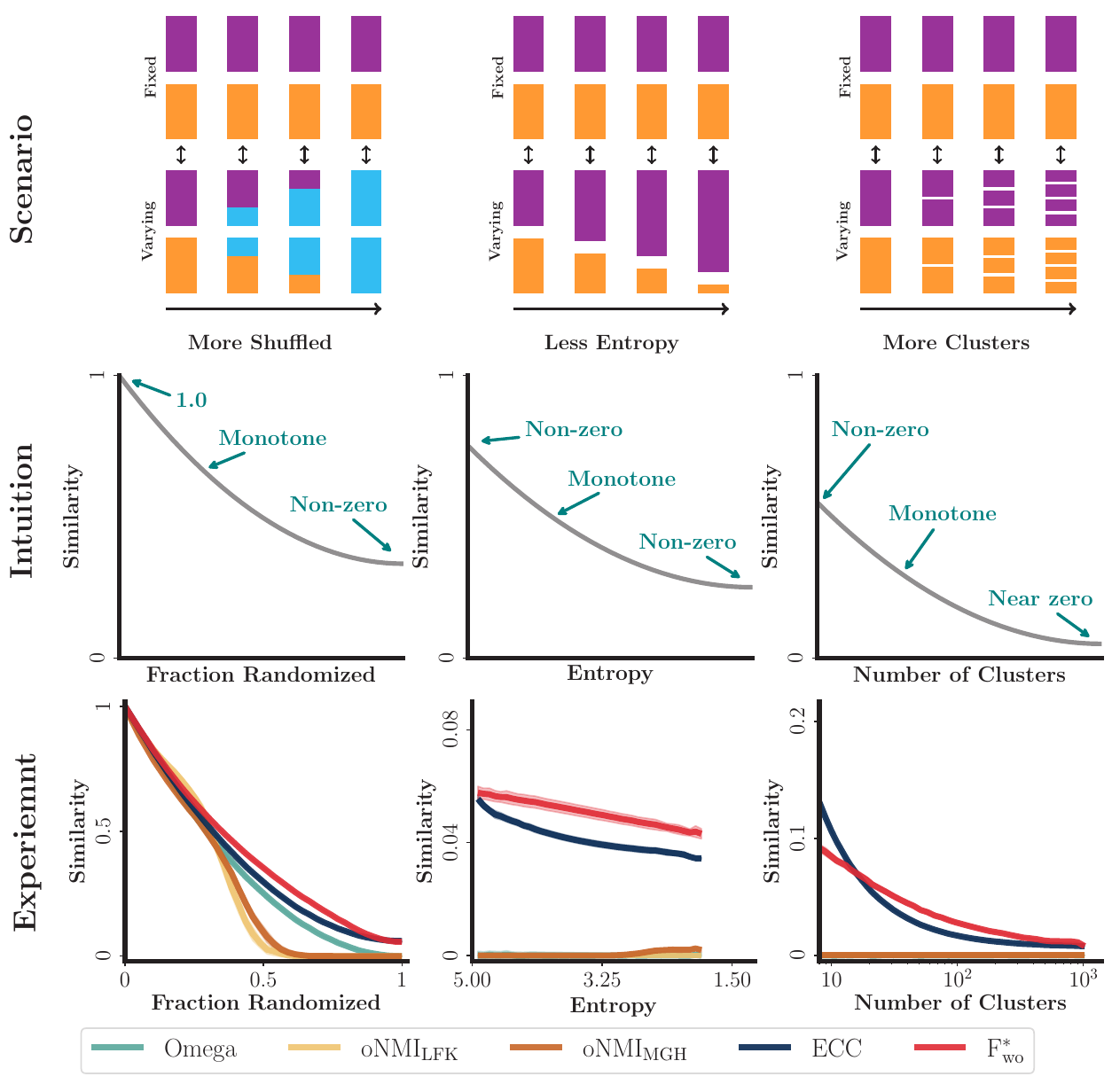}
    \caption{
        Recreation of the experiment in Figure 2 of \cite{ecc} with $F^*_{wo}$.
        Each column represents one scenario, including a visual of the scenario, the intuitive behaviour of a similarity measure in this scenario, and the actual behaviour of several measures in each row respectively.
        Each line is the average of $100$ simulations and the shaded region (usually too small to appear) covers plus or minus one standard deviation.
        In each scenario, the proposed $F^*_{wo}$ measure matches our intuition.
    }
    \label{fig:experiment}
\end{figure*}

Next, we test two more scenarios that involve clusterings with either outliers or overlapping clusters.
To generate clusterings, we use a method from the recently proposed \textit{Artificial Benchmark for Community Detection with Outliers and Overlaps} (ABCD+o$^2$)~\cite{abcdoo}.
This method leverages a low-dimensional geometric reference layer to create correlations between the overlapping clusters that are similar to those found in real world graph datasets.

As input, the method takes a set of initial cluster sizes $p_1, p_2, \dots p_m$ (such that for $n$ objects, $\sum_{1 \leq i \leq m} p_i = n$) and an expansion factor $\eta$ that corresponds to the average number of clusters per object.
First, the method creates a partition $P_1, P_2, \dots, P_m$ with the given sizes $p_1, p_2, \dots p_m$.
Each object is assigned an i.i.d.\ uniform random vector in the unit circle.
Next, for each $1 \leq i \leq m$, the object with the vector furthest from the origin that has not already been assigned to a part become the seed of part $i$.
The part $P_i$ is set as the seed object and its $p_i - 1$ nearest neighbours that have also not yet been assigned to a part.
In the second stage, each part is expanded around its center of mass until it reaches size $p_i \cdot \eta$, where $\eta$ is an input parameter governing the average number of clusters to which an object belongs.
For an exact description of the method, we refer to~\cite{abcdoo}.

In the fourth scenario, we test each comparison measure on overlapping clusters generated with the method described above.
For $25$ samples, we generate an underlying geometric layer and initial partition, and then compare various values of $\eta$ to a reference clustering where $\eta = 3$.
Since the second stage of the algorithm is deterministic, we expect perfect similarity when $\eta = 3$ and decreasing similarity as $|\eta - 3|$ increases.

In the fifth scenario, we reverse second stage of the method and shrink each cluster around its center of mass (this can be thought of as setting $\eta < 1$).
Since we start with a partition, there will be outliers, but no overlapping clusters.
We compare to a reference clustering corresponding to $\eta = 0.5$.
This time we test two out of five measures, as only $F^*_{wo}$ is explicitly defined for outliers, and ECS can accommodate outliers by treating them as isolated vertices in the cluster-induced element graph.

\begin{figure}[t]
    \centering
    \includegraphics[width=0.8\linewidth]{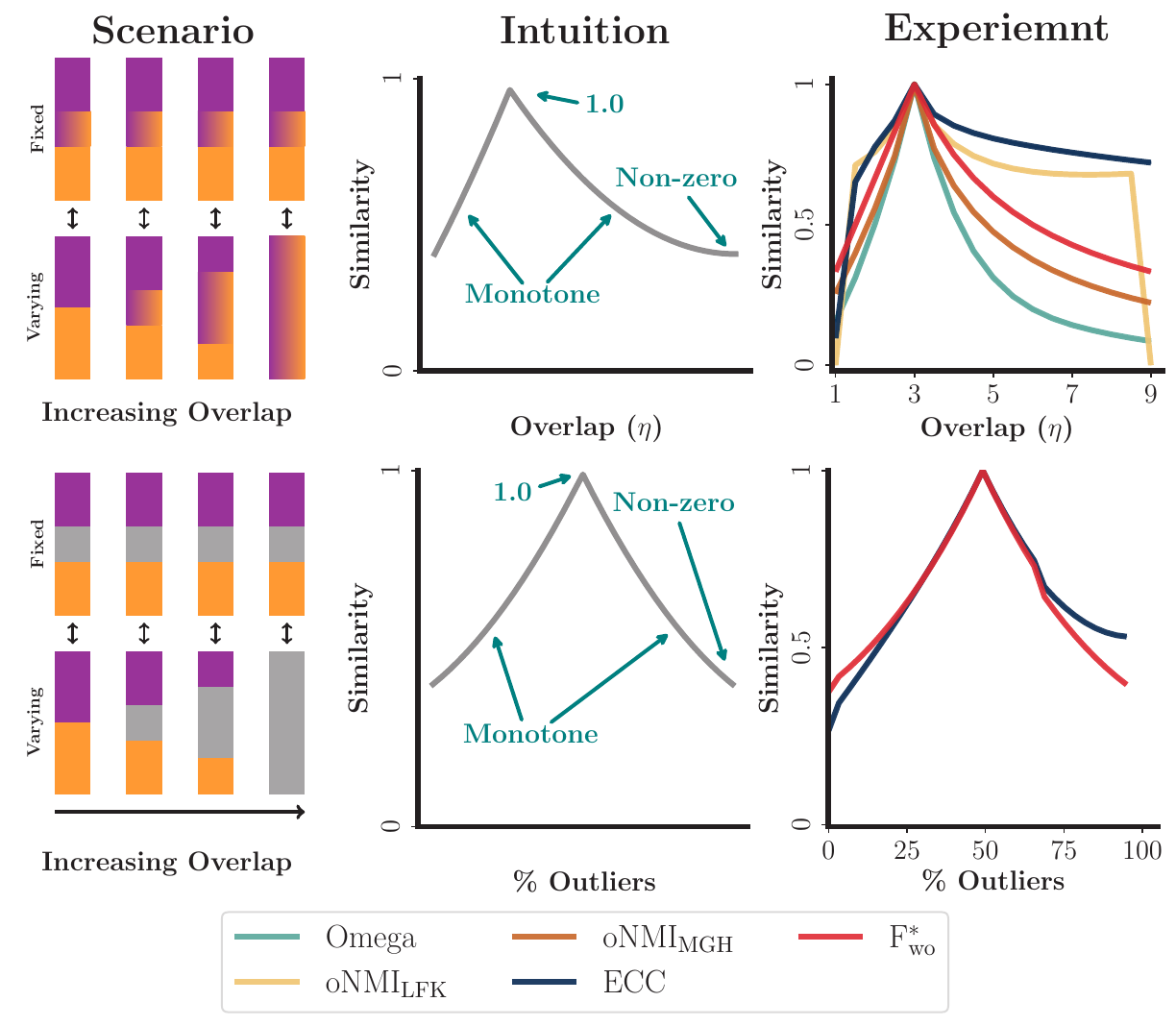}
    \caption{
        Comparing comparison measure behaviour on intuitive scenarios when overlaps or outliers are present.
        The top row is the scenario with overlapping clusters, and the second row has outliers; the columns, from left to right, show a visual of the scenario, the intuitive behaviour, and the actual behaviour of each method.
    }
    \label{fig:experiment2}
\end{figure}

Once again, as seen in Figure \ref{fig:experiment2}, $F^*_{wo}$ behaves as intuition would suggest.
The ECS also follows the intuition in both scenarios, although the similarity for highly overlapping clusterings (large $\eta$) is rather high.

\subsection{Graph Aware Clustering Comparison}
Since overlapping clusters are often necessary to fully capture the communities in complex networks~\cite{kaminski2021mining}, we provide a short discussion about incorporating the structure of the graph when comparing graph clusterings.
In particular, Poulin and Th\'{e}berge~\cite{gam} propose several graph-aware extensions of popular partition measures by transforming a partition of the vertices to a binary edge classification (inter-part versus intra-part).
In the case of overlapping clusters, we cannot use a binary edge classification, but we can create an edge clustering from a node clustering.
For a graph $G$ with vertices $V$ and edges $E$, and a clustering $\cC$ of the vertices, we define the edge clustering induced by $\cC$ as
\begin{equation*}
    \cE_G(\cC) = \{ \{uv \in E: u,v \in C\} \: \forall \: C \in \cC\}.
\end{equation*}
We can also define the vertex clustering induced by an edge clustering, as is commonly used by edge clustering algorithms~\cite{ahn2010},
\begin{equation*}
    \cC_G(\cE) = \{ \{v \in V: \exists \: uv \in C\} \: \forall \: C \in \cE\}.
\end{equation*}
When evaluating the performance of a graph clustering algorithm, we can compare clusterings from both the vertex and edge perspectives.
However, as could be the case when evaluating a linkage based edge clustering algorithm~\cite{ahn2010,hslc}, an edge cluster may span a vertex cluster while only covering a small fraction of the internal edges.
To remedy this, we can define the closure of an edge clustering $\bar{\cE}_G$ to add these internal edges back to the edge cluster, which can be computed as $\bar{\cE}_G = \cE_G(\cC_G(\cE))$.
When using a graph-aware measure for comparing an edge clustering algorithm to a ground truth vertex clustering, we recommend using the closed edge clustering since the internal edges will be included in the edge clustering induced by the ground truth.

\begin{figure}[!t]
    \centering
    \includegraphics[width=0.8\linewidth]{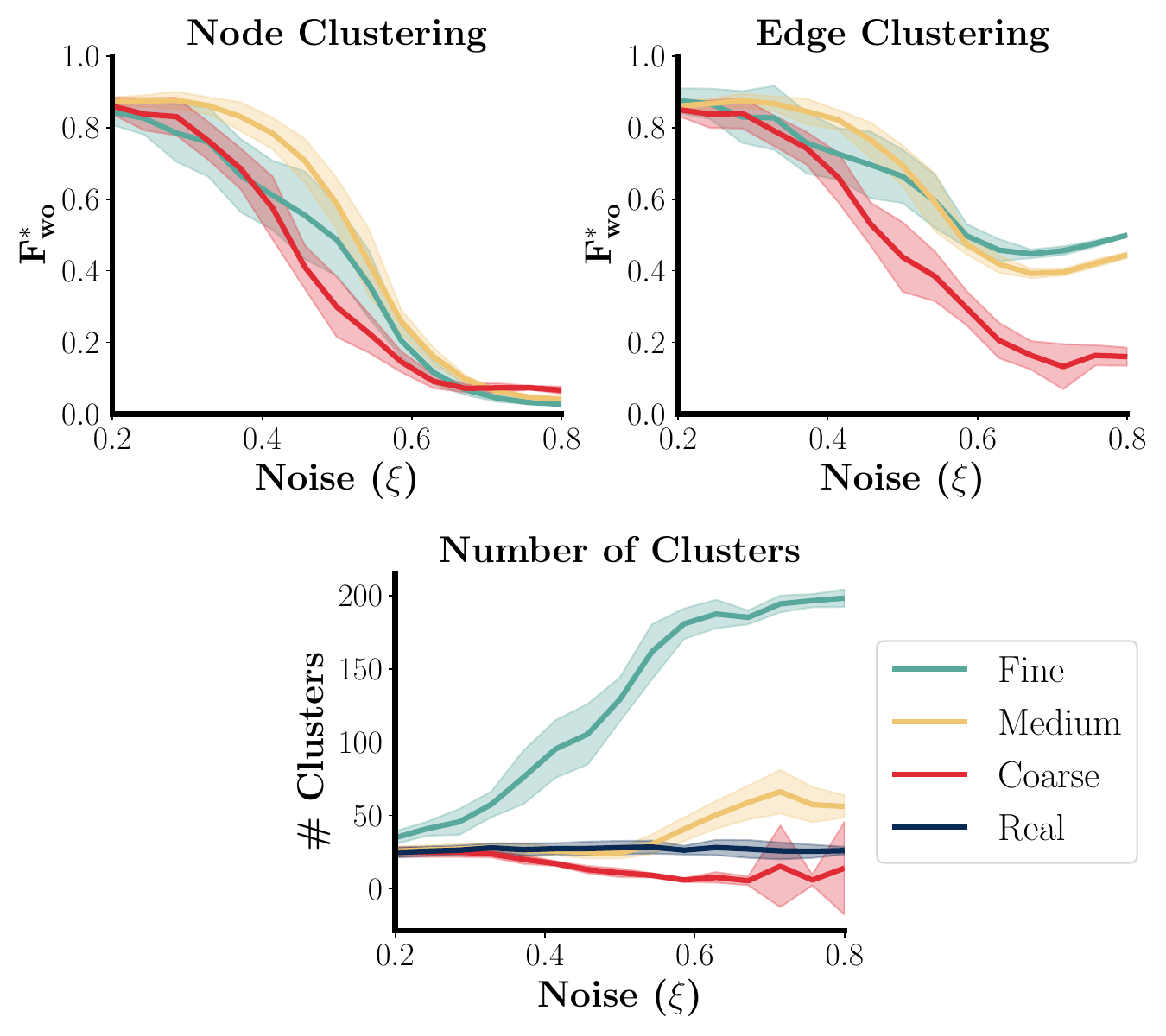}
    \caption{
        An experiment comparing graph aware clustering comparison, similar to that of Figure 5 from \cite{gam}.
        We use three resolutions ($0.5$, $1.0$, and $3.75$ shown in red, green, and yellow respectively) of the Leiden clustering algorithm on ABCD graphs with 2000 vertices and show, the similarity of each detected vertex clustering to the ground truth vertex clusters, the similarity of the detected induced edge clustering to the ground truth induced edge clusters, and the number of clusters
        The solid line represents the average over 10 samples and the shaded region covers plus or minus one standard deviation.
    }
    \label{fig:gam}
\end{figure}

In Figure~\ref{fig:gam}, we evaluate three resolutions of the Leiden~\cite{leiden} graph clustering algorithm on synthetic \textit{Artificial Benchmark for Community Detection} (ABCD)~\cite{abcd} graphs in an experiment similar to that in Figure~5 of~\cite{gam}.
The ABCD model, similar to the popular LFR model~\cite{lfr}, features power-law degrees and community sizes, but has a noise parameter $\xi \in [0,1]$ that controls the proportion of edges in the background graph and allows for a smooth transition from disjoint communities to a completely random graph.
The different resolutions of Leiden produce different cluster granularities, with $3.75$ giving a fine clustering (many small clusters), $0.5$ producing a coarse clustering (few large clusters), and the default resolution of $1.0$ producing clustering with medium-sized clusters.
While we see the same ranking from both vertex and edge perspectives for low noise (up to $\xi \approx 0.5$), the rankings are inverted for very high $\xi$.
In a very noisy graph, most of the edges are inter-cluster, and thus the ground truth edge clustering has many outliers.
Any fine partition will also have mostly inter-cluster edges, so it is similar to the ground truth from an edge clustering perspective, even if the vertex clusterings are not similar.
For evaluating graph clustering algorithms, we recommend analyzing performance from both perspectives to strengthen the conclusions.

\section{Conclusion}
In this note, we defined $F^*_{wo}$, a general method for quantifying the similarity between two clusterings of a common set of objects when the clusterings may have outliers or overlapping clusters.
The score has intuitive behaviour, is easy to compute, and does not exhibit some common biases, which makes it viable for evaluating clustering algorithms.
In Appendix \ref{sec:code}, we provide a simple yet efficient python implementation of $F^*_{wo}$.

\section*{Code Availability}
A python implementation with support for additional input types and graph aware comparisons is available at \href{https://github.com/ryandewolfe33/Fstar}{https://github.com/ryandewolfe33/Fstar}.
The repository also includes code used in the experiments of Section \ref{sec:experiments}.

\section*{Acknowledgements}
R.D.~acknowledges the support of the Natural Sciences and Engineering Research Council of Canada (NSERC) via a CGS-M scholarship.

\bibliographystyle{plainurl}
\bibliography{reference}

\pagebreak
\appendix
\section{Proof of Theorem \ref{thm:robust}} \label{app:proof}

\begin{thm*}[\textbf{\ref{thm:robust}}]
    Let $\cCone$ and $\cCtwo$ be clusterings of $X$, and let $\tcCone$ be a clustering created from $\cCone$ by moving one object, $x$, into or out of one cluster.
    Let $B = \sum_{C \in \cCtwo} |C|$.
    If $|\cCone| = |\tcCone|$ (no clusters were created or destroyed), say $x$ was added or removed from $C_1 \in \cCone$.
    Let $\gamma = |\{C \in \cCtwo : C \text{ matched with } C_1 \text{ or } x \in C\}|$.
    Then
    \begin{align*}
    &\left| F^*_{wo}(\cCone, \cCtwo) - F^*_{wo}(\tcCone, \cCtwo) \right|\\
    &\qquad = \bigO \left( \frac{1}{|X|} \max \left\{ 1, \frac{|X \setminus O^{(2)}| \gamma}{B} \right\} \right).
    \end{align*}
    Otherwise, let $\Gamma = |\{C \in \cCtwo : x \in C\}|$, and
    \begin{equation*}
            \left| F^*_{wo}(\cCone, \cCtwo) - F^*_{wo}(\tcCone, \cCtwo) \right| = \bigO \left( \frac{1}{|X|} \max \left\{ 1, \frac{|X \setminus O^{(2)}| \Gamma}{B} \right\} \right).
    \end{equation*}
\end{thm*}

First, we prove a few lemmas about some of the intermediate scores, and we will then combine these results to prove Theorem \ref{thm:robust}.
For rest of the section, let $\cCone$ and $\cCtwo$ be clusterings of $X$. 
Let $x \in X \setminus C_1$ for some $C_1 \in \cCone$, and create a new cluster $\tilde{C} = C_1 \cup \{x\}$.
Let $\tcCone$ be the clustering obtained from $\cCone$ by replacing $C_1$ with $\tilde{C}$.
This corresponds to the first scenario in Theorem \ref{thm:robust}.

\begin{lem}\label{lem:one}
    If $C_j \subseteq X$ then
    \begin{equation*}
        \left|F^*(C_1, C_j) - F^*(\tilde{C}, C_j)\right| \leq \frac{1}{|C_1|} \text{ and } \left|F^*(C_1, C_j) - F^*(\tilde{C}, C_j)\right| \leq \frac{1}{|C_j|}.
    \end{equation*}
\end{lem}
\begin{proof}
    There are two cases to consider. First, if $x \in C_j$, then
    \begin{equation*}
        \left|F^*(C_1, C_j) - F^*(\tilde{C}, C_j)\right| = \left| \frac{|C_1 \cap C_j|}{|C_1 \cup C_j|} - \frac{|\tilde{C} \cap C_j|}{|\tilde{C} \cup C_j|} \right| = \frac{1}{|C_1 \cup C_j|}.
    \end{equation*}
    And if $x \not \in C_j$, then
    \begin{equation*}
        \left|F^*(C_1, C_j) - F^*(\tilde{C}, C_j)\right| = \left| \frac{|C_1 \cap C_j|}{|C_1 \cup C_j|} - \frac{|\tilde{C}  \cap C_j|}{|\tilde{C} \cup C_j|} \right| = \frac{|C_1 \cap C_j|}{|C_1 \cup C_j| \cdot |\tilde{C} \cup C_j|} \leq \frac{1}{|\tilde{C} \cup C_j|}.
    \end{equation*}
    In either case, $|C_1|, |C_j| \leq |C_1 \cup C_j| \leq |\tilde{C} \cup C_j|$ so $1 / |\tilde{C} \cup C_j| \leq 1 / |C_1 \cup C_j| \leq 1/ |C_1|, 1 / |C_j|$. 
\end{proof}

\begin{lem*}[\textbf{\ref{lem:two}}]
    Let $A = \sum_{C \in \cCone} |C|$.
    \begin{equation*}
        \left| \hat{F}^*_w(\cCone, \cCtwo) - \hat{F}^*_w(\tcCone, \cCtwo) \right| = \bigO\left(\frac{1}{A}\right).
    \end{equation*}
\end{lem*}
\begin{proof}
    Recall that
    \begin{equation*}
        \hat{F}^*_w(\cCone, \cCtwo) = \frac{1}{A} \sum_{C \in \cCone} |C|F^*(C, \cCtwo).
    \end{equation*}
    Using the relationship between $\cCone$ and $\tcCone$, we get
    \begin{equation*}
        \hat{F}^*_w(\tcCone, \cCtwo) = \frac{1}{A+1}\left((|C_1|+1)F^*(C_1, \cCtwo) + \sum_{C_i \in \cCone, i \neq 1} |C_i|F^*(C_i,\cCtwo)\right).
    \end{equation*}
    It follows that
    \begin{align*}
    & \left| \hat{F}^*_w (\cCone, \cCtwo) - \hat{F}^*_w (\tcCone, \cCtwo) \right| \\
    \le& \left| \frac {|C_1| F^*(C_1, \cCtwo)} {A} - \frac {(|C_1|+1) F^*(\tilde{C}_1, \cCtwo)} {A+1} \right| \\
    & \qquad + \left| \frac {1}{A} - \frac {1}{A+1} \right| \cdot \left| \sum_{C_i \in \cCone, i \neq 1} |C_i| F^*(C_i, \cCtwo) \right| \\
    \le& \frac {A |C_1| \left| F^*(C_1, \cCtwo) - F^*(\tilde{C}_1, \cCtwo) \right| + \left| |C_1| F^*(C_1, \cCtwo) - A F^*(\tilde{C}_1, \cCtwo) \right| }{A \cdot (A+1)} \\
    & \qquad + \frac {1}{A \cdot (A+1)} \cdot A.
    \end{align*}
    Using Lemma~\ref{lem:one}, we get 
    \begin{eqnarray*}
    \left| \hat{F}^*_w (\cCone, \cCtwo) - \hat{F}^*_w (\tcCone, \cCtwo) \right| &\le& \frac {A+A+A}{A \cdot (A+1)} = \bigO\left(\frac{1}{A}\right),
    \end{eqnarray*}
    which finishes the proof.
\end{proof}

\begin{lem} \label{lem:three}
    Let $B = \sum_{C \in \cCtwo} |C|$ and $\gamma = \left| \{ C \in \cCtwo : C \text{ matched with } C_1 \text{ or } x \in C\} \right|$.
    \begin{equation*}
        \left| \hat{F}^*_w(\cCtwo, \cCone) - \hat{F}^*_w(\cCtwo, \tcCone) \right| \leq \frac{ \gamma }{B}.
    \end{equation*}
\end{lem}
\begin{proof}
    First, note that if $x \not \in C$, then $F^*(C, \tilde{C}) < F^*(C, C_1)$.
    Thus, any cluster not containing $x$ that does not match with $C_1$ will not match with $\tilde{C}$.
    So, if $ x \not \in C$ and $C$ is not matched with $C_1$ then $F^*(C, \cCone) = F^*(C, \tcCone)$.
    Thus, using Lemma~\ref{lem:one} we get
    \begin{align*}
        \left| \hat{F}^*_w(\cCtwo, \cCone) - \hat{F}^*_w(\cCtwo, \tcCone) \right| &= \left| \sum_{C \in \cCtwo} \frac{|C|}{B} \left( F^*(C, \cCone) - F^*(C, \tcCone) \right) \right| \\
        & \leq \sum_{C \in \cCtwo} \frac{|C|}{B} \left| F^*(C, \cCone) - F^*(C, \tcCone) \right| \\
        & \leq \frac{ \gamma }{B},
    \end{align*}
    and the proof is finished.
\end{proof}

Now, we can prove Theorem~\ref{thm:robust}.
\begin{proof}[Proof of Theorem~\ref{thm:robust}] 
    Due to symmetry, without loss of generality, we may assume one of the following four cases.
    \begin{enumerate}
        \item[1a.] $x \not \in O^{(1)}$ and is added to existing cluster $C_1$ to make clustering $\tcCone$.
        \item[1b.] $x \in O^{(1)}$ and is added to existing cluster $C_1$ to make clustering $\tcCone$.
        \item[2a.] $x \not \in O^{(1)}$ and makes a new cluster $\{x\}$ to make clustering $\hcCone$.
        \item[2b.] $x \in O^{(1)}$ and makes a new cluster $\{x\}$ to make clustering $\hcCone$.
    \end{enumerate}
    Recall that $A = \sum_{C \in \cCone} |C|$ and $B = \sum_{C \in \cCtwo} |C|$.

    \textbf{Case 1a}:
    \begin{align*}
        & \left| F^*_{wo}(\cCone, \cCtwo) - F^*_{wo}(\tcCone, \cCtwo) \right| \\
        = & \left| 0.5 \left( \frac{|X \setminus O^{(1)}|}{|X|} \right) \hat{F}^*_w(\cCone, \cCtwo) + 0.5 \left( \frac{| X \setminus O^{(2)}|}{|X|} \right) \hat{F}^*_w(\cCtwo, \cCone) \right. \\
        & \left. \qquad \qquad - 0.5 \left( \frac{|X \setminus O^{(1)}|}{|X|} \right) \hat{F}^*_w(\tcCone, \cCtwo) - 0.5 \left( \frac{|X \setminus O^{(2)}|}{|X|} \right) \hat{F}^*_w(\cCtwo, \tcCone) \right| \\
        \leq & \frac{|X \setminus O^{(1)}|}{|X|} \frac{1}{A} + \frac{|X \setminus O^{(2)}|}{|X|} \frac{ \gamma }{B},
    \end{align*}
    by Lemmas~\ref{lem:two} and~\ref{lem:three}.
    And, since $|X \setminus O^{(1)}| \leq A$, we conclude that
    \begin{equation*}
        \left| F^*_{wo}(\cCone, \cCtwo) - F^*_{wo}(\tcCone, \cCtwo) \right| \leq \bigO \left( \frac{1}{|X|} \max \left\{ 1, \frac{|X \setminus O^{(2)}| \gamma }{B} \right\} \right).
    \end{equation*}

    \textbf{Case 1b}:
    First, we consider each direction of $\hat{F}^*_{wo}$ individually.     
    \begin{align*}
        & \left| \hat{F}^*_{wo}(\cCone, \cCtwo) - \hat{F}^*_{wo}(\tcCone, \cCtwo) \right| \\
        = & \left| \frac{|O^{(1)}|}{|X|}F^*(O^{(1)}, O^{(2)}) + \frac{|X \setminus O^{(1)}|}{|X|}\hat{F}^*_w(\cCone, \cCtwo) \right. \\
        & \qquad \qquad \left. - \frac{|\tilde{O}^{(1)}|}{|X|}F^*(\tilde{O}^{(1)}, O^{(2)}) - \frac{|X \setminus \tilde{O}^{(1)}|}{|X|}\hat{F}^*_w(\tcCone, \cCtwo) \right| \\
        \leq & \frac{1}{|X|}F^*(O^{(1)}, O^{(2)}) + \frac{|O^{(1)}|-1}{|X|}\left|F^*(O^{(1)}, O^{(2)}) - F^*(\tilde{O}^{(1)}, O^{(2)})\right| \\
        & \qquad \qquad + \frac{1}{|X|}\hat{F}^*_w(\tcCone, \cCtwo) + \frac{|X \setminus O^{(1)}|}{|X|} \left|\hat{F}^*_w(\cCone, \cCtwo) - \hat{F}^*_w(\tcCone, \cCtwo) \right| \\
        = & \; \bigO \left( \frac{1}{|X|} \right).
    \end{align*}
    And for the other direction,
    \begin{align*}
        & \left| \hat{F}^*_{wo}(\cCtwo, \cCone) - \hat{F}^*_{wo}(\cCtwo, \tcCone) \right| \\
        \leq & \frac{|O^{(2)}|}{|X|}\left| F^*(O^{(2)}, O^{(1)}) - F^*(O^{(2)}, \tilde{O}^{(1)})\right| + \frac{|X \setminus O^{(2)}|}{|X|}\left| \hat{F}^*_w(\cCtwo, \cCone) - \hat{F}^*_w(\cCtwo, \tcCone)\right|\\
        = & \; \bigO \left( \frac{1}{|X|} \frac{|X \setminus O^{(2)}| \gamma}{B}\right).
    \end{align*}
    So together,
    \begin{equation*}
        \left| F^*_{wo}(\cCone, \cCtwo) - F^*_{wo}(\tcCone, \cCtwo) \right| = \bigO \left( \frac{1}{|X|} \max \left\{ 1, \frac{|X \setminus O^{(2)}| \gamma}{B} \right\} \right).
    \end{equation*}

    \textbf{Preliminaries for Case 2a and 2b}:
    For these cases, we need to slightly alter the proofs of Lemmas~\ref{lem:two} and \ref{lem:three} for when $x$ creates a new cluster.
    First, a similar bound to Lemma~\ref{lem:two}:
    \begin{align*}
        & \left| \hat{F}^*_{w}(\cCone, \cCtwo) - \hat{F}^*_{w}(\hcCone, \cCtwo) \right| \\
        = & \left| \sum_{C \in \cCone} \frac{|C|}{A} F^*(C, \cCtwo) - \sum_{C \in \cCone} \frac{|C|}{A+1} F^*(C, \cCtwo) - \frac{1}{A+1} F^*(\{x\}, \cCtwo) \right| \\
        \leq & \left( \sum_{C \in \cCone} \frac{1}{A(A+1)} + \frac{1}{A+1} F^*(\{x\}, \cCtwo)\right) \\
        = & \; \bigO \left( \frac{1}{A} \right).
    \end{align*}
    And also a similar bound to Lemma~\ref{lem:three}.
    In this case, if $x \not \in C$, then $F^*(C, \{x\}) = 0$ so $C$ will not match with $\{x\}$.
    Let $\Gamma = |\{C \in \cCtwo : x \in \cCtwo\}|$
    Thus, 
    \begin{align*}
        & \left|\hat{F}^*_w(\cCtwo, \cCone) - \hat{F}^*_w(\cCtwo, \hcCone)\right|\\
        \leq & \sum_{C \in \cCtwo} \frac{|C|}{B} \left| F^*(C, \cCone) - F^*(C, \hcCone) \right| \\
        \leq & \frac{\Gamma}{|B|}.
    \end{align*}

    \textbf{Case 2a}:
    \begin{align*}
        & \left| F^*_{wo}(\cCone, \cCtwo) - F^*_{wo}(\hcCone, \cCtwo) \right| \\
        \leq & \frac{|X \setminus O^{(1)}|}{|X|}\left|\hat{F}^*_{w}(\cCone, \cCtwo) - \hat{F}^*_{w}(\hcCone, \cCtwo) \right| + \frac{|X \setminus O^{(2)}|}{|X|} \left| \hat{F}^*_w(\cCtwo, \cCone) - \hat{F}^*_w(\cCtwo, \hcCone) \right| \\
        = & \bigO \left( \frac{1}{|X|} \max \left\{ 1, \frac{|X \setminus O^{(2)}| \Gamma }{B} \right\} \right).
    \end{align*}

    \textbf{Case 2b}:
    \begin{align*}
        & \left| F^*_{wo}(\cCone, \cCtwo) - F^*_{wo}(\hcCone, \cCtwo) \right| \\
        \leq & \frac{1}{|X|}F^*(O^{(1)}, O^{(2)}) + \frac{|O^{(1)}|-1}{|X|}\left|F^*(O^{(1)}, O^{(2)}) - F^*(\hat{O}^{(1)}, O^{(2)})\right| \\
        & \qquad + \frac{1}{|X|}\hat{F}^*_w(\hcCone, \cCtwo) + \frac{|X \setminus O^{(1)}|}{|X|} \left|\hat{F}^*_w(\cCone, \cCtwo) - \hat{F}^*_w(\hcCone, \cCtwo)\right| \\
        & \qquad + \frac{|O^{(2)}|}{|X|}\left|F^*(O^{(2)}, O^{(1)}) - F^*(O^{(2)}, \hat{O}^{(1)}) \right| + \frac{|X \setminus O^{(2)}|}{|X|}\left| \hat{F}^*_w(\cCtwo, \cCone) - \hat{F}^*_w(\cCtwo, \hcCone)\right| \\
        = & \; \bigO \left( \frac{1}{|X|} \max \left\{ 1, \frac{|X \setminus O^{(2)}| \Gamma}{B} \right\} \right).
    \end{align*}
\end{proof}

\pagebreak
\section{Python Implementation} \label{sec:code}
\begin{minted}{python}
import numpy as np
import scipy.sparse as sp


def compare(c1, c2):
    # Inputs c1, c2: The clusterings to compare.
    # Both must be scipy csr matrices of size (n_clusters, n_objects)
    # with c[i,j] = 1 if object j is in cluster i.
    if c1.shape[0] == 0 and c2.shape[0] == 0: # No clusters
        return 1
    c1_sizes = np.asarray(c1.sum(axis=1)).reshape(-1)
    c1_props = c1_sizes / np.sum(c1_sizes)
    c2_sizes = np.asarray(c2.sum(axis=1)).reshape(-1)
    c2_props = c2_sizes / np.sum(c2_sizes)
    intersect = c1.astype("float64") @ c2.transpose().astype("float64")
    row_sizes = sp.diags(c1_sizes)
    col_sizes = sp.diags(c2_sizes)
    nz = intersect.copy() # Store non-zero data indices
    nz.data[:] = 1
    union = row_sizes @ nz + nz @ col_sizes - intersect
    union.data = 1 / union.data
    fs = intersect.multiply(union)
    fs_l = fs.max(axis=1).toarray().reshape(-1)
    fs_lw = np.sum(fs_l * c1_props)
    fs_r = fs.max(axis=0).toarray().reshape(-1)
    fs_rw = np.sum(fs_r * c2_props)
    c1_o = c1.count_nonzero(axis=0) == 0
    c2_o = c2.count_nonzero(axis=0) == 0
    o_intersect = np.sum(c1_o * c2_o)
    o_fs = 0
    if o_intersect > 0:
        o_fs = o_intersect / (np.sum(c1_o) + np.sum(c2_o) - o_intersect)
    alpha = 1 - np.sum(c1_o) / len(c1_o)
    beta = 1 - np.sum(c2_o) / len(c2_o)
    fs_l_wo = (1 - alpha) * o_fs + alpha * fs_lw
    fs_r_wo = (1 - beta) * o_fs + beta * fs_rw
    return 0.5 * (fs_l_wo + fs_r_wo)
\end{minted}

\end{document}